\documentclass{article} 

\usepackage{etoolbox}
\newtoggle{withappendix}
\newtoggle{arxiv}
\toggletrue{withappendix}
\togglefalse{arxiv}
\toggletrue{withappendix}\toggletrue{arxiv}

\iftoggle{arxiv}{
  \usepackage{times}
  \setlength{\textwidth}{6.5in}
  \setlength{\textheight}{9in}
  \setlength{\oddsidemargin}{0in}
  \setlength{\evensidemargin}{0in}
  \setlength{\topmargin}{-0.5in}
  \newlength{\defbaselineskip}
  \setlength{\defbaselineskip}{\baselineskip}
  \setlength{\marginparwidth}{0.8in}
}{
  \usepackage{nips15submit_e,times}
}

\usepackage{hyperref}
\usepackage{url}

\iftoggle{arxiv}{
  \usepackage[numbers,sort]{natbib}
}{
  \usepackage[numbers,sort]{natbib}
}
\usepackage{multibib}
\newcites{sec}{Secondary Literature}

\patchcmd{\thebibliography}{\section*{\refname}}{}{}{}

\usepackage{amsopn,amsmath,amsthm,amssymb}

\usepackage{graphicx}

\usepackage{algorithm}
\usepackage{algorithmic}

\usepackage{bbold}

\usepackage{subfigure}

\usepackage{breqn}

\usepackage{tabu}

\usepackage{tikz}
\usetikzlibrary{shapes,backgrounds}

\usepackage{url,xspace}

\usepackage{xspace}

\usepackage{thmtools, thm-restate}

\definecolor{blue-violet}{rgb}{0.54, 0.17, 0.89}

\newcommand{\crcchange}[1]{#1}

\theoremstyle{definition}
\declaretheorem{definition}
\theoremstyle{plain}
\declaretheorem{lemma}
\declaretheorem{theorem}
\declaretheorem{corollary}

\declaretheorem{statement}

\newcommand{\R}{\mathbb{R}}

\newcommand{\Abs}[1]{\left| #1 \right| }
\newcommand{\Prob}[2][]{\underset{#1}{\mathbf{P}}\left( \hiderel{#2} \right) }
\newcommand{\Probc}[2]{\mathbf{P}\left( \hiderel{#1} \middle | \hiderel{#2} \right) }
\newcommand{\Exv}[1]{\mathbf{E}\left[ #1 \right]}
\newcommand{\Exvc}[2]{\mathbf{E}\left[ #1 \middle | \hiderel{#2} \right]}

\newcommand{\Exvsmall}[1]{\mathbf{E}[ #1 ]}
\newcommand{\Exvcsmall}[2]{\mathbf{E}[ #1 | \hiderel{#2} ]}
\newcommand{\normfsmall}[1]{\| #1 \|_F}
\newcommand{\normsmall}[1]{\| #1 \|}

\newcommand{\F}{\mathcal{F}}

\newcommand{\normf}[1]{\left\| #1 \right\|_F}
\newcommand{\norm}[1]{\left\| #1 \right\|}

\newcommand{\buckwild}{\textsc{Buckwild!}\xspace}
\newcommand{\hogwild}{\textsc{Hogwild!}\xspace}

\DeclareMathOperator*{\plog}{log}

\title{Taming the Wild: A Unified Analysis of \hogwild-Style Algorithms}

\iftoggle{arxiv}{
  \author{
  Christopher De Sa$^\dagger$,
  Ce Zhang$^\ddagger$,
  Kunle Olukotun$^\dagger$, and
  Christopher R{\'e}$^\dagger$ \\
  $\dagger$ Stanford University,
  $\ddagger$ University of Wisconsin-Madison \\
  \texttt{cdesa@stanford.edu},
  \texttt{czhang@cs.wisc.edu}, \\
  \texttt{kunle@stanford.edu},
  \texttt{chrismre@stanford.edu}
  }
}
{
  \author{
  Christopher De Sa,
  Ce Zhang,
  Kunle Olukotun, and
  Christopher R{\'e} \\
  \texttt{cdesa@stanford.edu},
  \texttt{czhang@cs.wisc.edu}, \\
  \texttt{kunle@stanford.edu},
  \texttt{chrismre@stanford.edu} \\
  Departments of Electrical Engineering and Computer Science\\
  Stanford University, Stanford, CA 94309
  }
}

\iftoggle{arxiv}{}{
\nipsfinalcopy 
}

\begin{document}

\maketitle

\begin{abstract}
Stochastic gradient descent (SGD) is a ubiquitous algorithm for a
variety of machine learning problems. Researchers and industry have
developed several techniques to optimize SGD's runtime performance,
including asynchronous execution and reduced precision. Our main
result is a martingale-based analysis that enables us to capture 
the rich
noise models that may arise from such techniques. Specifically, we use
our new analysis in three ways: (1) we derive convergence rates for
the convex case (\hogwild) with relaxed assumptions on the sparsity of
the problem; (2) we analyze asynchronous SGD algorithms for
non-convex matrix problems including matrix completion; and (3) we
design and analyze an asynchronous SGD algorithm, called \buckwild,
that uses lower-precision arithmetic. We show experimentally that our
algorithms run efficiently for a variety of problems on modern
hardware.
\end{abstract}

\section{Introduction}
Many problems in machine learning can be written as a
stochastic optimization problem
\[
  \begin{array}{llll}
    \mbox{minimize} & \Exvsmall{\tilde f(x)} &
    \mbox{over} & x \in \R^n,
  \end{array}
\]
where $\tilde f$ is a random objective function.  One popular method
to solve this is with stochastic gradient descent (SGD), an iterative
method which, at each timestep $t$, chooses a random objective sample
$\tilde f_t$ and updates
\begin{equation}
  \label{eqnSGDupdate}
  x_{t+1} = x_t - \alpha \nabla \tilde f_t(x_t),
\end{equation}
where $\alpha$ is the step size.  For most problems, this update step
is easy to compute, and perhaps because of this SGD is a ubiquitous
algorithm with a wide range of applications in machine
learning~\cite{bottou2010large}, including neural network
backpropagation
\cite{lecun-98x,bottou-bousquet-2008,bottou2012stochastic},
recommendation systems \cite{parambath2013matrix,Gupta:2013}, and
optimization~\cite{rakhlin2011making}.  For non-convex problems, SGD
is popular---in particular, it is widely used in deep learning---but its
success is poorly understood theoretically.

Given SGD's success in industry, practitioners have developed methods to
speed up its computation. One popular method to
speed up SGD and related algorithms is using asynchronous execution.
In an asynchronous algorithm, such as
\hogwild~\cite{recht2011hogwild}, multiple threads run an update rule
such as Equation~\ref{eqnSGDupdate} in parallel without
locks. \hogwild and other lock-free algorithms have been applied to a
variety of uses, including PageRank approximations
(FrogWild!~\cite{FrogWild}), deep learning (Dogwild!~\cite{DogWild})
and recommender systems~\cite{yu2012scalable}. 
\crcchange{Many asynchronous
versions of other stochastic algorithms have been individually analyzed, such
as stochastic coordinate
descent (SGD)~\cite{liu2013asynchronous,liu2015asynchronous} and
accelerated parallel proximal coordinate descent
(APPROX)~\cite{fercoq2013accelerated}, producing rate
results that are similar to those of \hogwild
Recently, \citet{gupta2015deep} gave an empirical analysis of the effects
of a low-precision variant of SGD on neural network training.
Other} variants of stochastic algorithms have been
proposed~\cite{konecny2014s2cd,tao2012stochastic,johansson2009randomized,
duchi2012randomized,richtarik2012parallel,tappenden2015complexity};
only a fraction of these algorithms have been analyzed in the 
asynchronous case.
Unfortunately, a new variant of SGD (or a related algorithm) may
violate the assumptions of existing analysis, and hence there are
gaps in our understanding of these techniques.

One approach to filling this gap is to analyze each purpose-built
extension from scratch: an entirely new model for each type of
asynchrony, each type of precision, etc. In a practical sense,
this may be unavoidable, but ideally there would be a single technique
that could analyze many models. In this vein, we prove a
martingale-based result that enables us to treat many different
extensions as different forms of noise within a unified 
model. We demonstrate our technique with three results:
\begin{enumerate}
  \item For the convex case, \hogwild requires strict sparsity
    assumptions. Using our techniques, we are able to relax these
    assumptions and still derive convergence rates. Moreover, under
    \hogwild's stricter assumptions, we recover the previous
    convergence rates.
  \item We derive convergence results for an asynchronous SGD
    algorithm for a non-convex matrix completion problem. We derive
    the first rates for asynchronous SGD following the recent
    (synchronous) non-convex SGD work of \citet{desa2014global}.
  \item We derive convergence rates in the presence of
    quantization errors such as those introduced by fixed-point
    arithmetic. We validate our results experimentally, and show
    that \buckwild{} can achieve speedups of up to 
    $2.3 \times$ over \hogwild-based algorithms
    for logistic regression.
\end{enumerate}
One can combine these different methods both theoretically and
empirically. We begin with our main result, which describes our 
martingale-based approach and our model.

\section{Main Result}
Analyzing asynchronous algorithms is challenging because, unlike in the
sequential case where there is a single copy of the iterate $x$, in the
asynchronous case each core has a separate copy of $x$ in its own cache.
Writes from one core may take some time to be propagated to another core's
copy of $x$, which results in race conditions where stale data is used to
compute the gradient updates.  This difficulty is compounded in the non-convex
case, where a series of unlucky random events---bad initialization,
inauspicious steps, and race conditions---can cause the algorithm to get stuck
near a saddle point or in a local minimum.  

Broadly, we analyze algorithms that repeatedly update $x$ by 
running an update step
\begin{equation}
  \label{eqnGeneralUpdate}
  x_{t+1} = x_t - \tilde G_t(x_t),
\end{equation}
for some i.i.d. update function $\tilde G_t$.
For example, for SGD, we would have $G(x) = \alpha \nabla \tilde f_t(x)$.
The goal of the algorithm must be to produce an iterate in some
\emph{success region} $S$---for example, a ball centered at the optimum $x^*$.
For any $T$, after running the algorithm for $T$ timesteps, we say that
the algorithm has \emph{succeeded} if $x_t \in S$ for some
$t \le T$; otherwise, we say that the algorithm has \emph{failed}, and we
denote this failure event as $F_T$.

Our main result is a technique that allows us to bound the convergence rates
of asynchronous SGD and related algorithms, even for some non-convex problems.
We use martingale methods, which have produced elegant
convergence rate results for both convex and some
non-convex~\cite{desa2014global} algorithms.  Martingales enable us to
model multiple forms of error---for example, from stochastic sampling, random
initialization, and asynchronous delays---within a single statistical model.
Compared to standard techniques, they also allow us to analyze algorithms that
sometimes get stuck, which is useful for non-convex problems.
Our core contribution is that
a martingale-based proof for the convergence of a sequential stochastic
algorithm can be easily modified 
to give a convergence rate for an asynchronous version.

A \emph{supermartingale}~\cite{fleming1991} is a stochastic process $W_t$ such
that $\Exvcsmall{W_{t+1}}{W_t} \le W_t$.  That is, the expected value is
\crcchange{non-increasing} over time.
A martingale-based proof of convergence for the sequential version of this
algorithm must construct a supermartingale
$W_t(x_t, x_{t-1}, \ldots, x_0)$ that is a 
function of both the time and the current and past iterates; this function
informally represents how unhappy we are with the current state of the
algorithm.  Typically, it will have the following properties.
\begin{definition}
  For a stochastic algorithm as described above,
  a non-negative process $W_t: \R^{n \times t} \rightarrow \R$ is a
  \emph{rate supermartingale}
  with horizon $B$ if the following conditions are true.
  First, it must be a supermartingale; that is,
  for any sequence $x_t, \ldots, x_0$ and any $t \le B$,
  \begin{equation}
    \label{eqnBoundedW}
    \Exvsmall{W_{t+1}(x_t - \tilde G_t(x_t), x_t, \ldots, x_0)}
    \le
    W_t(x_t, x_{t-1}, \ldots, x_0).
  \end{equation}
  Second, for all times $T \le B$ and for any sequence $x_T, \ldots, x_0$,
  if the algorithm has not succeeded by time $T$ (that is, $x_t \notin S$
  for all $t < T$), it must hold that
  \begin{equation}
    \label{eqnBoundedTime}
    W_T(x_T, x_{T-1}, \ldots, x_0) \ge T.
  \end{equation}
  This represents the fact that we are unhappy with running
  for many iterations without success.
\end{definition}
Using this, we can easily bound the convergence rate of 
the sequential version of the algorithm.
\begin{statement}
  \label{stmtSequential}
  Assume that we run a sequential stochastic algorithm, for which $W$ is a
  \emph{rate supermartingale}.  For any $T \le B$, the probability
  that the algorithm has not succeeded by time $T$ is
  \[
    \Prob{F_T} \le \frac{\Exvsmall{W_0(x_0)}}{T}.
  \]
\end{statement}
\begin{proof}
  In what follows, we let $W_t$ denote the actual value taken on by the
  function in a process defined by (\ref{eqnGeneralUpdate}).  That is,
  $W_t = W_t(x_t, x_{t-1}, \ldots, x_0)$.
  By applying (\ref{eqnBoundedW}) recursively, for any $T$,
  \[
    \Exvsmall{W_T} \le \Exvsmall{W_0} = \Exvsmall{W_0(x_0)}.
  \]
  By the law of total expectation applied to the failure event $F_T$,
  \[
    \Exvsmall{W_0(x_0)}
    \ge
    \Exvsmall{W_T}
    = 
    \Prob{F_T} \Exvcsmall{W_T}{F_T}
    +
    \Prob{\lnot F_T} \Exvcsmall{W_T}{\lnot F_T}.
  \]
  Applying (\ref{eqnBoundedTime}), i.e. $\Exvcsmall{W_T}{F_T} \ge T$,
  and recalling that $W$ is nonnegative results in
  \[
    \Exvsmall{W_0(x_0)}
    \ge
    \Prob{F_T} T;
  \]
  rearranging terms produces the result in Statement \ref{stmtSequential}.
\end{proof}
This technique is very general; in subsequent sections we
show that rate supermartingales can be constructed for SGD on all convex
problems and for some algorithms for non-convex problems.

\subsection{Modeling Asynchronicity}
\label{ssHardwareModel}
The behavior of an asynchronous SGD algorithm depends both on the
problem it is trying to solve and on the hardware
it is running on.  For ease of analysis, we assume that the
hardware has the following characteristics.  These are basically the same
assumptions used to prove the original \hogwild
result~\cite{recht2011hogwild}.
\begin{itemize}
  \item There are multiple threads running iterations of
    (\ref{eqnGeneralUpdate}),
    each with their own cache.  At any point in time, these caches may hold
    different values for the variable $x$, and they communicate via some cache
    coherency protocol.
  \item There exists a central store $\mathcal{S}$ (typically RAM) at which all
    writes are serialized.  This provides a consistent value for the state of
    the system at any point in real time.
  \item If a thread performs a read $R$ of a previously written value $X$,
    and then writes another value $Y$ (dependent on $R$), then the write that
    produced $X$ will be committed to $\mathcal{S}$ before the write that
    produced $Y$.
  \item Each write from an iteration of (\ref{eqnGeneralUpdate}) is to only a
    single entry of $x$ and is done using an
    atomic read-add-write instruction.  That is, there are no write-after-write
    races (handling these is possible, but complicates the analysis).
\end{itemize}
Notice that, if we let $x_t$ denote the value of the vector $x$ in the central
store $\mathcal{S}$ after $t$ writes have occurred, then since the writes are
atomic, the value of $x_{t+1}$ is solely dependent on the single thread that
produces the write that is serialized next in $\mathcal{S}$.  If
we let $\tilde G_t$ denote the update function sample that is used by that
thread for that write, and $v_t$ denote the cached value of $x$ used by that
write, then
\begin{equation}
  \label{eqnHogwildUpdate}
  x_{t+1} = x_t - \tilde G_t(\tilde v_t)
\end{equation}
Our hardware model further constrains the value of $\tilde v_t$: all the read
elements of $\tilde v_t$ must have been written to $\mathcal{S}$ at some
time before $t$.  Therefore, for some nonnegative variable $\tilde \tau_{i,t}$,
\begin{equation}
  \label{eqnHogwildVDef}
  e_i^T \tilde v_t = e_i^T x_{t - \tilde \tau_{i,t}},
\end{equation}
where $e_i$ is the $i$th standard basis vector.  We can think of
$\tilde \tau_{i,t}$ as the \emph{delay} in the $i$th coordinate caused by the
parallel updates.

We can conceive of this system as a stochastic process
with two sources of randomness: the noisy update function
samples $\tilde G_t$ and the delays $\tilde \tau_{i,t}$. We assume that the
$\tilde G_t$
are independent and identically distributed---this is reasonable because they
are sampled independently by the updating threads.  It would be unreasonable,
though, to assume the same for the $\tilde \tau_{i,t}$, since delays may very
well be correlated in the system.  Instead, we assume that the delays are
bounded from above by some random variable $\tilde \tau$.  Specifically,
if $\F_t$, the \emph{filtration},
denotes all random events that occurred before timestep $t$, then for
any $i$, $t$, and $k$,
\begin{equation}
  \label{eqnHogwildTauBound}
  \Probc{\tilde \tau_{i,t} \ge k}{\F_t}
  \le
  \Prob{\tilde \tau \ge k}.
\end{equation}
We let $\tau = \Exvsmall{\tilde \tau}$, and call $\tau$ the
\emph{worst-case expected delay}.

\subsection{Convergence Rates for Asynchronous SGD}
\label{ssAsyncConvergence}
Now that we are equipped with a stochastic model for the asynchronous SGD
algorithm, we show how we can use a rate supermartingale to give a convergence
rate for asynchronous algorithms.  To do this, we need some continuity and
boundedness assumptions; we collect these into a definition, and then state
the theorem.
\begin{definition}
  An algorithm with rate supermartingale $W$ is $(H, R, \xi)$-bounded if the
  following conditions hold.
  First, $W$ must be Lipschitz continuous in the current iterate
  with parameter $H$; that is, for any $t$, $u$, $v$, and sequence
  $x_t, \ldots, x_0$,
  \begin{equation}
    \label{eqnWLipschitz}
    \normsmall{
      W_t(u, x_{t-1}, \ldots, x_0)
      - 
      W_t(v, x_{t-1}, \ldots, x_0)
    } 
    \le 
    H \normsmall{u - v}.
  \end{equation}
  Second, $\tilde G$ must be Lipschitz continuous in expectation
  with parameter $R$; that is, for any $u$, and $v$,
  \begin{equation}
    \label{eqnGLipschitz}
    \Exvsmall{\normsmall{
      \tilde G(u)
      - 
      \tilde G(v)
    }}
    \le
    R \normsmall{u - v}_1.
  \end{equation}
  Third, the expected magnitude of the update must be bounded by $\xi$.
  That is, for any $x$,
  \begin{equation}
    \label{eqnBoundedUpdateDist}
    \Exvsmall{\normsmall{\tilde G(x)}} \le \xi.
  \end{equation}
\end{definition}

\begin{theorem}
  \label{thmHogwild}
  Assume that we run an asynchronous stochastic algorithm with the above
  hardware model, for which $W$ is a $(H, R, \xi)$-bounded 
  \emph{rate supermartingale} with horizon $B$.
  Further assume that $H R \xi \tau < 1$.
  For any $T \le B$, the probability
  that the algorithm has not succeeded by time $T$ is
  \[
    \Prob{F_T} \le \frac{\Exvsmall{W(0, x_0)}}{(1 - H R \xi \tau) T}.
  \]
\end{theorem}
Note that this rate depends only on the worst-case expected delay $\tau$ and
not on any other properties of the hardware model.
Compared to the result of Statement \ref{stmtSequential}, the probability of
failure has only increased by a factor of $1 - H R \xi \tau$.
In most practical cases, $H R \xi \tau \ll 1$, so this increase in
probability is negligible.

Since the proof of this theorem is simple, but uses non-standard techniques,
we outline it here.  First, notice that the process $W_t$, which was
a supermartingale in the sequential case, is not in the 
asynchronous case because of the delayed updates.  Our strategy is to
use $W$ to produce a new process $V_t$ that is a supermartingale in this
case.  For any $t$ and $x_{\cdot}$, if $x_u \notin S$ for all $u < t$, we
define
\[
  V_t(x_t, \ldots, x_0)
  =
  W_t(x_t, \ldots, x_0)
  -
  H R \xi \tau t
  +
  H R
  \sum_{k=1}^{\infty}
  \norm{x_{t-k+1} - x_{t-k}}
  \sum_{m=k}^{\infty}
  \Prob{\tilde \tau \ge m}.
\]
Compared with $W$, there are two additional terms here.  The first term
is negative, and cancels out some of the unhappiness from
(\ref{eqnBoundedTime}) that we ascribed to running for many iterations.
We can interpret this as us accepting that we may need to run for more
iterations than in the sequential case.
The second term measures the distance between recent iterates; we would be
unhappy if this becomes large because then the noise from the delayed updates
would also be large.  On the other hand, if $x_u \in S$ for some $u < t$, then
we define
\[
  V_t(x_t, \ldots, x_u, \ldots, x_0)
  =
  V_u(x_u, \ldots, x_0).
\]
We call $V_t$ a \emph{stopped process} because its value doesn't change after
success occurs.
It is straightforward to show that $V_t$ is a
supermartingale for the asynchronous algorithm.
Once we know this, the same
logic used in the proof of Statement \ref{stmtSequential} can be used to prove
Theorem \ref{thmHogwild}.

Theorem \ref{thmHogwild} gives us a straightforward way of bounding the
convergence time of any asynchronous stochastic algorithm.  First, we
find a rate supermartingale for the problem; this is typically no harder than
proving \crcchange{sequential convergence.}
Second, we \crcchange{find parameters}
such that the problem is
$(H, R, \xi)$-bounded, \crcchange{typically}
; this is easily done for well-behaved problems
by using differentiation to bound the
Lipschitz constants.  Third, we apply Theorem \ref{thmHogwild} to get a rate
for asynchronous SGD.  Using this method, analyzing an asynchronous algorithm
is really no more difficult than analyzing its sequential analog.

\section{Applications}
Now that we have proved our main result, we turn our attention to
applications.  We show, for a couple of algorithms, how to construct a 
rate supermartingale.  We demonstrate that doing this allows us to recover
known rates for \hogwild algorithms as well as analyze cases where no
known rates exist.

\subsection{Convex Case, High Precision Arithmetic}
\label{ssConvexHighPrecision}
First, we consider the simple case of using asynchronous SGD to minimize
a convex function $f(x)$ using unbiased gradient samples $\nabla \tilde f(x)$.
That is, we run the update rule
\begin{equation}
  \label{eqnConvexUpdate}
  x_{t+1} = x_t - \alpha \nabla \tilde f_t(x).
\end{equation}
We make the standard assumption
that $f$ is strongly convex with parameter $c$; that is, for all $x$ and $y$
\begin{equation}
  \label{eqnStrongConvexity}
  (x - y)^T \left( \nabla f(x) - \nabla f(y) \right) \ge c \normsmall{x - y}^2.
\end{equation}
We also assume continuous differentiability of $\nabla \tilde f$ with
1-norm Lipschitz constant $L$, 
\begin{equation}
  \label{eqnConvexLipschitz}
  \Exvsmall{\normsmall{\nabla \tilde f(x) - \nabla \tilde f(y)}}
  \le
  L \normsmall{x - y}_1.
\end{equation}
We require that the second moment of the gradient sample is also
bounded for some $M > 0$ by
\begin{equation}
  \label{eqnConvexNoise}
  \Exvsmall{\normsmall{\nabla \tilde f(x)}^2} \le M^2.
\end{equation}
For some $\epsilon > 0$, we let the success region be
\[
  S = \{ x | \normsmall{x - x^*}^2 \le \epsilon \}.
\]
Under these conditions, we can construct a rate supermartingale for this
algorithm.
\begin{lemma}
  \label{lemmaConvexW}
  There exists a $W_t$ where, if the algorithm hasn't succeeded by timestep
  $t$,
  \[
    W_t(x_t, \ldots, x_0)
    =
    \frac{
      \epsilon
    }{
      2 \alpha c \epsilon
      -
      \alpha^2 M^2
    }
    \log\left( e \norm{x_t - x^*}^2 \epsilon^{-1} \right)
    +
    t,
  \]
  such that $W_t$ is a rate submartingale for the above algorithm
  with horizon $B = \infty$.  Furthermore, it is $(H, R, \xi)$-bounded with
  parameters: 
  \crcchange{
  $H = 2 \sqrt{\epsilon} (2 \alpha c \epsilon - \alpha^2 M^2)^{-1}$,
  $R = \alpha L$, and $\xi = \alpha M$.}
\end{lemma}
Using this and Theorem \ref{thmHogwild} gives us a direct bound on
the failure rate of convex \hogwild SGD.
\begin{corollary}
  \label{corConvexHogwild}
  Assume that we run an asynchronous version of the above SGD algorithm,
  where for some constant $\vartheta \in (0, 1)$ we 
  choose step size
  \[
    \alpha
    =
    \frac{c \epsilon \vartheta}{M^2 + 2 L M \tau \sqrt{\epsilon}}.
  \]
  Then for any $T$, the probability
  that the algorithm has not succeeded by time $T$ is
  \[
    \Prob{F_T}
    \le
    \frac{
      M^2 + 2 L M \tau \sqrt{\epsilon}
    }{
      c^2 \epsilon \vartheta T
    }
    \log\left( e \norm{x_0 - x^*}^2 \epsilon^{-1} \right).
  \]
\end{corollary}
This result is
more general than the result in ~\citet{recht2011hogwild}.
The main differences are: that we make no
assumptions about the sparsity structure of
the gradient samples; and that our rate depends only on the second moment
of $\tilde G$ and the expected value of $\tilde \tau$, as opposed to requiring
absolute bounds on their magnitude.
Under their stricter assumptions, the
result of Corollary \ref{corConvexHogwild} recovers their rate.

\subsection{Convex Case, Low Precision Arithmetic}
\label{ssConvexLowPrecision}
One of the ways \buckwild achieves high performance is by using low-precision
fixed-point arithmetic.  This introduces additional noise to the system in the
form of \emph{round-off error}.  We consider this error to be part of the
\buckwild hardware model.  We assume that the round-off error can be
modeled by an unbiased rounding function
operating on the update samples.  That is, for some chosen precision factor
$\kappa$, there is a random quantization function
$\tilde Q$
such that, for any $x \in \R$, it holds that $\mathbf{E}[\tilde Q(x)] = x$, and
the round-off error is bounded by $|\tilde Q(x) - x | < \alpha \kappa M$.
Using this function, we can write a low-precision asynchronous update rule
for convex SGD as
\begin{equation}
  \label{eqnBuckwildUpdate}
  x_{t+1}
  =
  x_t - \tilde Q_t\left(\alpha \nabla \tilde f_t(\tilde v_t)\right),
\end{equation}
where $\tilde Q_t$ operates only on the single nonzero entry of
$\nabla \tilde f_t(\tilde v_t)$.
In the same way as we did in the high-precision case, we can use these
properties to construct a rate supermartingale for the low-precision version
of the convex SGD algorithm, and then use Theorem \ref{thmHogwild} to 
bound the failure rate of convex \buckwild
\begin{corollary}
  \label{corConvexBuckwild}
  Assume that we run asynchronous low-precision convex SGD,
  and for some $\vartheta \in (0, 1)$, we choose step size
  \[
    \alpha
    =
    \frac{
      c \epsilon \vartheta
    }{
      M^2 (1 + \kappa^2) + L M \tau (2 + \kappa^2) \sqrt{\epsilon}
    },
  \]
  then for any $T$, the probability
  that the algorithm has not succeeded by time $T$ is
  \[
    \Prob{F_T}
    \le
    \frac{
       M^2 (1 + \kappa^2) + L M \tau (2 + \kappa^2) \sqrt{\epsilon}
    }{
      c^2 \epsilon \vartheta T
    }
    \log\left( e \norm{x_0 - x^*}^2 \epsilon^{-1} \right).
  \]
\end{corollary}
Typically, we choose a precision such that $\kappa \ll 1$; in this case,
the increased error compared to the result of Corollary \ref{corConvexHogwild}
will be negligible and we will converge in a number of samples
that is very similar to the high-precision, sequential case.
\crcchange{Since each \buckwild update runs in less time than an equivalent
\hogwild update, this result means that an execution of \buckwild will produce
same-quality output in less wall-clock time compared with \hogwild}

\subsection{Non-Convex Case, High Precision Arithmetic}
\label{ssNonConvex}
Many machine learning problems are non-convex, but are still solved in practice
with SGD.  In this section, we show that our technique can be adapted to
analyze non-convex problems. Unfortunately, there are no general convergence
results that provide rates for SGD on non-convex problems, so it would be
unreasonable to expect a general proof of convergence for non-convex
\hogwild  Instead, we focus on a particular problem, low-rank least-squares
matrix completion,
\begin{equation}
  \label{eqnBMProblem}
  \begin{array}{ll}
    \mbox{minimize} & \Exvsmall{\normfsmall{\tilde A - x x^T}^2} \\
    \mbox{subject to} & x \in \R^n,
  \end{array}
\end{equation}
for which there exists a sequential SGD algorithm with a martingale-based rate
that has already been proven.  This problem arises in
general data analysis, subspace tracking,
principle component analysis, recommendation systems,
and other applications~\cite{desa2014global}.
In what follows, we let $A = \Exv{\tilde A}$.  We assume that $A$ is symmetric,
and has unit eigenvectors $u_1, u_2, \ldots, u_n$ with corresponding
eigenvalues $\lambda_1 > \lambda_2 \ge \cdots \ge \lambda_n$.  We let $\Delta$,
the \emph{eigengap}, denote $\Delta = \lambda_1 - \lambda_2$.

\citet{desa2014global} provide a martingale-based rate of convergence for a
particular SGD algorithm, Alecton, running on this problem. For simplicity,
we focus on only the rank-1 version of the \crcchange{problem,
and} we assume that, at each timestep, a single entry of
$A$ is used as a sample.  Under these conditions, Alecton uses the update rule
\begin{equation}
  \label{eqPowerIterUpdate}
  x_{t+1}
  =
  (I + \eta n^2 e_{\tilde i_t} e_{\tilde i_t}^T A
    e_{\tilde j_t} e_{\tilde j_t}^T ) x_t,
\end{equation}
where $\tilde i_t$ and $\tilde j_t$ are randomly-chosen indices in $[1, n]$.
It initializes $x_0$ uniformly on the sphere of some radius centered at
the origin.
We can equivalently think of this as a stochastic power iteration algorithm.
For any $\epsilon > 0$, we define the \emph{success set} $S$ to be
\begin{equation}
  \label{eqnAlectonDelta}
  S = \{ x | (u_1^T x)^2 \ge (1 - \epsilon) \norm{x}^2 \}.
\end{equation}
That is, we are only concerned with the direction of $x$, not its magnitude;
this algorithm only recovers the dominant eigenvector of $A$, not
its eigenvalue.
In order to show convergence for this entrywise sampling scheme,
\citet{desa2014global} require that the
matrix $A$ satisfy a \emph{coherence bound}~\cite{jain2012}.
\begin{definition}
  A matrix $A \in \R^{n \times n}$ is incoherent with parameter $\mu$ if
  for every standard basis vector $e_j$,
  and for all unit eigenvectors $u_i$ of the matrix,
  $(e_j^T u_i)^2 \le \mu^2 n^{-1}$.
\end{definition}
They also require that the step size be set, for some constants
$0 < \gamma \le 1$ and $0 < \vartheta < (1 + \epsilon)^{-1}$ as
\[
  \eta
  =
  \frac{
    \Delta \epsilon \gamma \vartheta
  }{
    2 n \mu^4 \normf{A}^2
  }.
\]
For ease of analysis, we add the additional assumptions that our algorithm runs
in some bounded space.  That is, for some
constant $C$, at all times $t$, $1 \le \norm{x_t}$ and $\norm{x_t}_1 \le C$.
As in the convex case, 
by following the martingale-based approach of \citet{desa2014global}, we are
able to generate a rate supermartinagle for this algorithm---to save space,
we only state its initial value and not the full expression.
\begin{lemma}
  \label{lemmaAlectonW}
  For the problem above, choose any
  horizon $B$ such that $\eta \gamma \epsilon \Delta B \le 1$.  
  Then there exists a function $W_t$ such that
  \crcchange{$W_t$ is a rate supermartingale for the above non-convex SGD
  algorithm with parameters 
  $H = 8 n \eta^{-1} \gamma^{-1} \Delta^{-1} \epsilon^{-\frac{1}{2}}$,
  $R = \eta \mu \normf{A}$, and $\xi = \eta \mu \normf{A} C$, and}
  \[
    \Exv{W_0(x_0)}
    \le
    2 \eta^{-1} \Delta^{-1}
    \log( e n \gamma^{-1} \epsilon^{-1} )
    +
    B \sqrt{2 \pi \gamma}.
  \]
\end{lemma}
Note that the analysis parameter $\gamma$ allows us to trade off between $B$,
which determines how long we can run the algorithm, and the initial value of
the supermartingale $\Exv{W_0(x_0)}$.
We can now produce a corollary about the convergence rate
by applying Theorem \ref{thmHogwild} and setting $B$ and $T$ appropriately.
\begin{corollary}
  \label{thmHogwildAlecton}
  Assume that we run \hogwild Alecton under these conditions for
  $T$ timesteps, as defined below.  Then the probability of failure,
  $\Prob{F_T}$, will be bounded as below.
  \begin{align*}
    T
    &=
    \frac{
      4 n \mu^4 \normf{A}^2
    }{
      \Delta^2 \epsilon \gamma \vartheta \sqrt{2 \pi \gamma}
    }
    \log \left(
      \frac{e n}{\gamma \epsilon}
    \right),
    &
    \Prob{F_T}
    &\le
    \frac{
      \sqrt{8 \pi \gamma} \mu^2
    }{
      \mu^2 - 4 C \vartheta \tau \sqrt{\epsilon}
    }.
  \end{align*}
\end{corollary}
The fact that we are able to use our technique to analyze a non-convex
algorithm illustrates its generality.  Note that it is possible to combine
our results to analyze asynchronous low-precision non-convex SGD, but the
resulting formulas are complex, so we do not include them here.

\section{Experiments}

We validate our theoretical results for both asynchronous
non-convex matrix completion and \buckwild, a \hogwild implementation
with lower-precision arithmetic. Like \hogwild, a \buckwild
algorithm has multiple threads running an update rule
(\ref{eqnGeneralUpdate}) in parallel
without locking. Compared with \hogwild, which uses 32-bit floating point
numbers
to represent input data, \buckwild uses limited-precision arithmetic
by rounding the input data to 8-bit or 16-bit integers. This not only
decreases the memory usage, but also allows us to take advantage of 
single-instruction-multiple-data (SIMD) instructions 
for integers on modern CPUs.

\begin{table}[t]%
\caption{Training loss of SGD as a function of arithmetic
precision for logistic regression.}%
\label{tabTrainingLoss}%
\begin{center}
\begin{tabu}{r|[2pt]c|c|c|[2pt]c|c|c}
Dataset & Rows & Columns & Size & 32-bit float & 16-bit int & 8-bit int \\
\tabucline[2pt]{-}
Reuters & 8K & 18K & 1.2GB & $0.5700$ & $0.5700$ & $0.5709$ \\
Forest  & 581K & 54 & 0.2GB & $0.6463$ & $0.6463$ & $0.6447$ \\
RCV1 & 781K & 47K & 0.9GB & $0.1888$ & $0.1888$ & $0.1879$ \\
Music & 515K    & 91    & 0.7GB      & $0.8785$ & $0.8785$ & $0.8781$ \\
\end{tabu}
\end{center}
\end{table}


\begin{figure}[t]
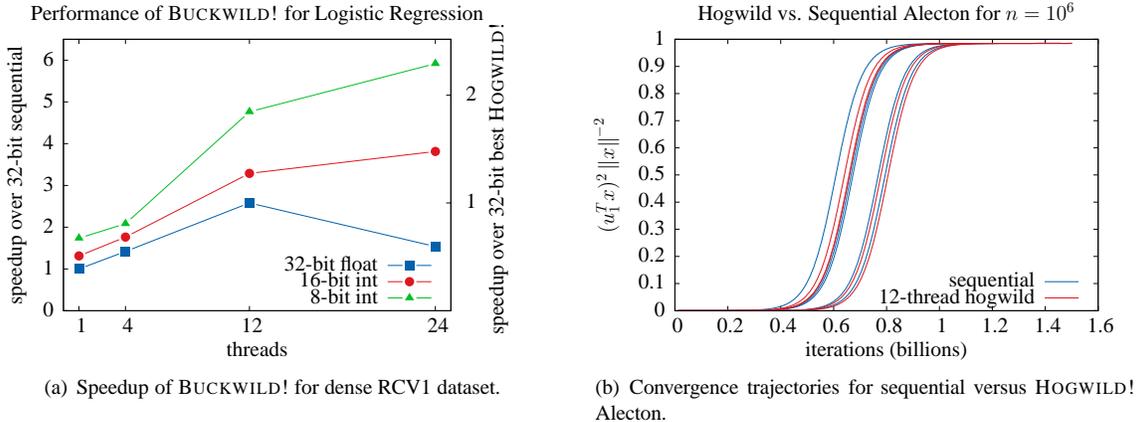
%
\iftoggle{arxiv}{}{\vspace{-2em}}
\centering
\subfigure[Speedup of \buckwild for dense RCV1 dataset.]{%
\centering%
\resizebox{!}{.30\textwidth}{
\Large \input{plotspeedup.tex}%
}
\label{figSpeedup}}\qquad%
\subfigure[Convergence trajectories for sequential versus \hogwild Alecton.]{%
\centering%
\resizebox{!}{.30\textwidth}{
\Large \input{plot1m.tex}%
}
\label{figHogwildAlecton}}
\label{figExperiments} 
\caption{
Experiments compare the training loss, performance, and convergence of
\hogwild and \buckwild algorithms with sequential and/or high-precision
versions.
}
\end{figure}

We verified our main claims by running \hogwild and \buckwild algorithms on
the discussed applications.
Table \ref{tabTrainingLoss} shows how the training loss of SGD for logistic
regression, a convex problem,
 varies as the precision is changed.  We ran SGD with step size
$\alpha = 0.0001$; however, results are similar across a range of step sizes.
We analyzed all four datasets reported in
DimmWitted~\cite{dimmwitted2014}
that favored \hogwild:
Reuters and RCV1, which are text classification datasets; Forest, which
arises from remote sensing; and Music, which is a music classification dataset.  
We implemented all GLM models reported in DimmWitted, including SVM, Linear
Regression,
and Logistic Regression, and report Logistic Regression because other
models have similar performance. The results illustrate that there is almost no
increase in training loss as the precision is decreased for these problems.
\crcchange{We also investigated 4-bit and 1-bit computation: the former was slower
than 8-bit due to a lack of 4-bit SIMD instructions, and the latter discarded
too much information to produce good quality results.}

Figure \ref{figSpeedup} displays the speedup of \buckwild running on the
dense-version of the RCV1 dataset compared to both full-precision sequential 
SGD (left axis) and best-case \hogwild (right axis).
Experiments ran on a machine with
two Xeon X650 CPUs, each with six hyperthreaded cores, and 24GB of RAM.
This plot illustrates
that incorporating low-precision arithmetic into our algorithm allows
us to achieve significant speedups over both sequential
and \hogwild SGD.  
\crcchange{(Note that we don't get full linear speedup because we are bound by
the available memory bandwidth; beyond this limit, adding additional threads
provides no benefits while increasing conflicts and thrashing the L1 and L2
caches.)}
This result, combined with the data in
Table~\ref{tabTrainingLoss}, suggest that by doing low-precision
asynchronous
updates, we can get speedups of up to $2.3 \times$ \crcchange{on these sorts
of datasets} without a significant increase in error.

Figure \ref{figHogwildAlecton} compares the convergence trajectories of
\hogwild and sequential versions of
the non-convex Alecton matrix completion algorithm on a synthetic data matrix
$A \in \R^{n \times n}$ with ten random eigenvalues $\lambda_i > 0$.
Each plotted series represents a different run of Alecton; the trajectories
differ somewhat because of the randomness of the algorithm.
The plot shows that the sequential and asynchronous versions behave
qualitatively similarly, and converge to the same noise floor.  For this
dataset, sequential Alecton took $6.86$ seconds to run while 12-thread
\hogwild Alecton took $1.39$ seconds, a $4.9 \times$ speedup.

\section{Conclusion}
This paper presented a unified theoretical framework for producing results
about the convergence rates of asynchronous and low-precision random
algorithms such as stochastic gradient descent.  We showed how a
martingale-based rate of convergence for a sequential, full-precision algorithm
can be easily leveraged to give a rate for an asynchronous, low-precision
version.  We also introduced \buckwild, a strategy for SGD that is able to
take advantage of modern hardware resources for both task and data parallelism,
and showed that it achieves near linear parallel speedup over sequential
algorithms.


\subsection*{Acknowledgments} 
{
\footnotesize

The \buckwild name arose out of conversations with Benjamin Recht.
Thanks also to Madeleine Udell for helpful conversations.

The authors acknowledge the support of:
DARPA FA8750-12-2-0335;
NSF IIS-1247701;
NSF CCF-1111943;
DOE 108845;
NSF CCF-1337375;
DARPA FA8750-13-2-0039; 
NSF IIS-1353606; ONR N000141210041
and N000141310129; NIH U54EB020405; Oracle; NVIDIA; Huawei; SAP Labs;
Sloan Research
Fellowship; Moore Foundation; American Family
Insurance; Google; and Toshiba.

}



\begingroup
\renewcommand{\section}[2]{\subsubsection*{#2}}
\small
\bibliographystyle{plainnat} 
\bibliography{references}
\endgroup

\iftoggle{withappendix}{

\clearpage

\appendix

\section{Proof of Theorem \ref{thmHogwild}}
\begin{proof}[Proof of Theorem \ref{thmHogwild}]
This proof is a more detailed version of the argument outlined in 
Section \ref{ssAsyncConvergence}.
First, we restate the definition of the process $V_t$ from the body of the
paper.  As long as the algorithm hasn't succeeded yet,
\begin{dmath*}
  V_t(x_t, \ldots, x_0)
  =
  W_t(x_t, \ldots, x_0)
  -
  H R \xi \tau t
  +
  H R
  \sum_{k=1}^{\infty}
  \norm{x_{t-k+1} - x_{t-k}}
  \sum_{m=k}^{\infty}
  \Prob{\tilde \tau \ge m}.
\end{dmath*}
At the next timestep, we will have $x_{t+1} = x_t + \tilde G(\tilde v_t)$,
and so
\begin{dmath*}
  V_{t+1}(x_t + \tilde G(\tilde v_t), x_t, \ldots, x_0)
  =
  W_{t+1}(x_t + \tilde G(\tilde v_t), x_t, \ldots, x_0)
  -
  H R \xi \tau (t + 1)
  +
  H R
  \norm{\tilde G(\tilde v_t)}
  \sum_{m=1}^{\infty}
  \Prob{\tilde \tau \ge m}
  +
  H R
  \sum_{k=2}^{\infty}
  \norm{x_{t-k+2} - x_{t-k+1}}
  \sum_{m=k}^{\infty}
  \Prob{\tilde \tau \ge m}.
\end{dmath*}
Re-indexing the second sum and applying the definition of $\tau$ produces
\begin{dmath*}
  V_{t+1}(x_t + \tilde G(\tilde v_t), x_t, \ldots, x_0)
  =
  W_{t+1}(x_t + \tilde G(\tilde v_t), x_t, \ldots, x_0)
  -
  H R \xi \tau (t + 1)
  +
  H R \tau
  \norm{\tilde G(\tilde v_t)}
  +
  H R
  \sum_{k=1}^{\infty}
  \norm{x_{t-k+1} - x_{t-k}}
  \sum_{m=k+1}^{\infty}
  \Prob{\tilde \tau \ge m}.
\end{dmath*}
Applying the Lipschitz continuity assumption (\ref{eqnWLipschitz}) for $W$
results in
\begin{dmath*}
  V_{t+1}(x_t + \tilde G(\tilde v_t), x_t, \ldots, x_0)
  \le
  W_{t+1}(x_t + \tilde G(x_t), x_t, \ldots, x_0)
  +
  H \norm{\tilde G(\tilde v_t) - \tilde G(x_t)}
  -
  H R \xi \tau (t + 1)
  +
  H R \tau
  \norm{\tilde G(\tilde v_t)}
  +
  H R
  \sum_{k=1}^{\infty}
  \norm{x_{t-k+1} - x_{t-k}}
  \sum_{m=k+1}^{\infty}
  \Prob{\tilde \tau \ge m}.
\end{dmath*}
Taking the expected value of both sides produces
\begin{dmath*}
  \Exv{V_{t+1}(x_t + \tilde G(\tilde v_t), x_t, \ldots, x_0)}
  \le
  \Exv{W_{t+1}(x_t + \tilde G(x_t), x_t, \ldots, x_0)}
  +
  H \Exv{\norm{\tilde G(\tilde v_t) - \tilde G(x_t)}}
  -
  H R \xi \tau (t + 1)
  +
  H R \tau
  \Exv{\norm{\tilde G(\tilde v_t)}}
  +
  H R
  \sum_{k=1}^{\infty}
  \norm{x_{t-k+1} - x_{t-k}}
  \sum_{m=k+1}^{\infty}
  \Prob{\tilde \tau \ge m}.
\end{dmath*}
Applying the rate supermartingale property (\ref{eqnBoundedW}) of $W$,
\begin{dmath*}
  \Exv{V_t(x_t + \tilde G(\tilde v_t), x_t, \ldots, x_0)}
  \le
  W_t(x_t, \ldots, x_0)
  +
  H \Exv{\norm{\tilde G(\tilde v_t) - \tilde G(x_t)}}
  -
  H R \xi \tau (t + 1)
  +
  H R \tau
  \Exv{\norm{\tilde G(\tilde v_t)}}
  +
  H R
  \sum_{k=1}^{\infty}
  \norm{x_{t-k+1} - x_{t-k}}
  \sum_{m=k+1}^{\infty}
  \Prob{\tilde \tau \ge m}.
\end{dmath*}
Applying the Lipschitz continuity assumption (\ref{eqnGLipschitz}) for 
$\tilde G$,
\begin{dmath*}
  \Exv{V_t(x_t + \tilde G(\tilde v_t), x_t, \ldots, x_0)}
  \le
  W_t(x_t, \ldots, x_0)
  +
  H R \Exv{\norm{\tilde v_t - x_t}_1}
  -
  H R \xi \tau (t + 1)
  +
  H R \tau
  \Exv{\norm{\tilde G(\tilde v_t)}}
  +
  H R
  \sum_{k=1}^{\infty}
  \norm{x_{t-k+1} - x_{t-k}}
  \sum_{m=k+1}^{\infty}
  \Prob{\tilde \tau \ge m}.
\end{dmath*}
Finally, applying the update distance bound (\ref{eqnBoundedUpdateDist}),
\begin{dmath*}
  \Exv{V_t(x_t + \tilde G(\tilde v_t), x_t, \ldots, x_0)}
  \le
  W_t(x_t, \ldots, x_0)
  +
  H R \Exv{\norm{\tilde v_t - x_t}_1}
  -
  H R \xi \tau (t + 1)
  +
  H R \xi \tau
  +
  H R
  \sum_{k=1}^{\infty}
  \norm{x_{t-k+1} - x_{t-k}}
  \sum_{m=k+1}^{\infty}
  \Prob{\tilde \tau \ge m}
  =
  W_t(x_t, \ldots, x_0)
  -
  H R \xi \tau t
  +
  H R
  \sum_{k=1}^{\infty}
  \norm{x_{t-k+1} - x_{t-k}}
  \sum_{m=k}^{\infty}
  \Prob{\tilde \tau \ge m}
  +
  H R \Exv{\norm{\tilde v_t - x_t}_1}
  -
  H R
  \sum_{k=1}^{\infty}
  \norm{x_{t-k+1} - x_{t-k}}
  \Prob{\tilde \tau \ge k}
  =
  V_t(x_t, \ldots, x_0)
  +
  H R \Exv{\norm{\tilde v_t - x_t}_1}
  -
  H R
  \sum_{k=1}^{\infty}
  \norm{x_{t-k+1} - x_{t-k}}
  \Prob{\tilde \tau \ge k}.
\end{dmath*}
Now, by the definition of the $\tilde v_t$, 
\begin{dmath*}
  \norm{\tilde v_t - x_t}_1
  =
  \sum_{i=1}^n \Abs{e_i^T x_t - e_i^T \tilde v_t}
  =
  \sum_{i=1}^n \Abs{e_i^T x_t - e_i^T x_{t - \tilde \tau_{i,t}}}
  \le
  \sum_{i=1}^n \sum_{k=1}^{\tilde \tau_{i,t}}
  \Abs{e_i^T x_{t-k+1} - e_i^T x_{t-k}}
\end{dmath*}
Furthermore, using the bound on $\tilde \tau_{i,t}$ from
(\ref{eqnHogwildTauBound}) gives us
\begin{dmath*}
  \Exv{\norm{\tilde v_t - x_t}_1}
  \le
  \sum_{i=1}^n \sum_{k=1}^{\infty}
  \Abs{e_i^T x_{t-k+1} - e_i^T x_{t-k}}
  \Prob{\tilde \tau_{i,t} \ge k}
  \le
  \sum_{i=1}^n \sum_{k=1}^{\infty}
  \Abs{e_i^T x_{t-k+1} - e_i^T x_{t-k}}
  \Prob{\tilde \tau \ge k}
  =
  \sum_{k=1}^{\infty}
  \norm{x_{t-k+1} - x_{t-k}}_1
  \Prob{\tilde \tau \ge k}
  =
  \sum_{k=1}^{\infty}
  \norm{x_{t-k+1} - x_{t-k}}
  \Prob{\tilde \tau \ge k},
\end{dmath*}
where the 1-norm is equal to the 2-norm here because each step only updates a
single entry of $x$.  Substituting this result in to the above equation
allows us to conclude that, if the algorithm hasn't succeeded by time $t$,
\begin{equation}
  \label{eqnVtProofSupermartingale}
  \Exv{V_t(x_t + \tilde G(\tilde v_t), x_t, \ldots, x_0)}
  \le
  V_t(x_t, \ldots, x_0).
\end{equation}
On the other hand, if it has succeeded, this statement will be vacuously true,
since $V_t$ does not change after success occurs.  Therefore,
(\ref{eqnVtProofSupermartingale}) will hold for all times.

In what follows, as in the proof of Statement \ref{stmtSequential},
we let $V_t$ denote the actual value taken on by the function during execution
of the algorithm.  That is, $V_t = V_t(x_t, x_{t-1}, \ldots, x_0)$.
By applying (\ref{eqnVtProofSupermartingale}) recursively, for any $T < B$,
we can show that
\[
  \Exv{V_T} \le \Exv{V_0}.
\]
Since we assumed as part of our hardware model that $x_t = x_0$ for $t < 0$,
\[
  \Exv{V_0} = \Exv{W_0(x_0)}.
\]
Therefore, by the law of total expectation
\begin{dmath*}
  \Exv{W_0(x_0)}
  \ge
  \Exv{V_T}
  =
  \Exvc{V_T}{F_T} \Prob{F_T}
  +
  \Exvc{V_T}{\lnot F_T} \Prob{\lnot F_T}
  \ge
  \Exvc{V_T}{F_T} \Prob{F_T}
  =
  \Exvc{
    W_T(x_T, \ldots, x_0)
    -
    H R \xi \tau T
    +
    H R
    \sum_{k=1}^{\infty}
    \norm{x_{T-k+1} - x_{T-k}}
    \sum_{m=k}^{\infty}
    \Prob{\tilde \tau \ge m}
  }{F_T} \Prob{F_T}
  \ge
  \left(
    \Exvc{W_T(x_T, \ldots, x_0)}{\F_T}
    -
    H R \xi \tau T
  \right) \Prob{F_T}.
\end{dmath*}
Since $W_t$ is a rate supermartingale, we can apply (\ref{eqnBoundedTime})
to get
\begin{dmath*}
  \Exv{W_0(x_0)}
  \ge
  \left(
    T
    -
    H R \xi \tau T
  \right) \Prob{F_T},
\end{dmath*}
and solving for $\Prob{F_T}$ produces
\[
  \Prob{F_T}
  \le
  \frac{
    \Exv{W_0(x_0)}
  }{
    (1 - H R \xi \tau) T
  },
\]
as desired.
\end{proof}

\section{Proofs for Convex Case}
\label{ssProofConvex}

First, we state the rate supermartingale lemma for the low-precision convex
SGD algorithm.

\begin{lemma}
  \label{lemmaConvexLowPrecisionW}
  There exists a $W_t$ with
  \[
    W_0(x_0)
    \le
    \frac{
      \epsilon
    }{
      2 \alpha c \epsilon
      -
      \alpha^2 M^2 (1 + \kappa^2)
    }
    \log\left( \frac{e \norm{x_0 - x^*}^2}{\epsilon} \right)
  \]
  such that $W_t$ is a rate submartingale for the above convex SGD algorithm
  with horizon $B = \infty$.  Furthermore, it is $(H, R, \xi)$-bounded with
  parameters: $R = \alpha L$, $\xi^2 = \alpha^2 (1 + \kappa^2) M^2$, and
  \[
    H
    =
    \frac{
      2 \sqrt{\epsilon}
    }{
      2 \alpha c \epsilon
      -
      \alpha^2 M^2 (1 + \kappa^2)
    }.
  \]
\end{lemma}

We note that, including this Lemma, the results in Section
\ref{ssConvexHighPrecision} are the
same as the results in Section \ref{ssConvexLowPrecision}, except that the
quantization factor is set as $\kappa = 0$.  It follows that it is sufficient
to prove only the Lemma and Corollary in \ref{ssConvexLowPrecision};
this is what we will do here.

In order to prove the results in this section, we will need some definitions
and lemmas, which we state now.
\begin{definition}[Piecewise Logarithm]
  \label{defnPiecewiseLogarithm}
  For the purposes of this document, we define the \emph{piecewise logarithm}
  function to be
  \[
    \plog(x)
    =
    \left\{
      \begin{array}{lr}
        \log(e x) & : x \ge 1 \\
        x & : x \le 1
      \end{array}
    \right.
  \]
\end{definition}
\begin{lemma}
  \label{lemmaPiecewiseLogarithm}
  The piecewise logarithm function is
  differentiable and concave.  Also, if $x \ge 1$, then for any $\Delta$,
  \[
    \plog(x (1 + \Delta)) \le \plog(x) + \Delta.
  \]
\end{lemma}
\begin{proof}
  The first part of the lemma follows from the fact that $\plog(x)$ is
  a piecewise function, where the pieces are both increasing and concave,
  and the fact that the function is differentiable at $x = 1$.  The second
  part of the lemma follows from the fact that a first-order approximation
  always overestimates a concave function.
\end{proof}

Armed with this definition, we prove Lemma \ref{lemmaConvexLowPrecisionW}.

\begin{proof}[Proof of Lemma \ref{lemmaConvexLowPrecisionW}]
  First, we note that, at any timestep $t$, if we evaluate the distance
  to the optimum at the next timestep using (\ref{eqnConvexUpdate}), then
  \begin{dmath*}
    \norm{x_t + \tilde G_t(x_t) - x^*}^2
    =
    \norm{x_t - x^*}^2
    -
    2 (x_t - x^*)^T \tilde Q_t \left( \alpha \nabla \tilde f_t(x_t) \right)
    +
    \norm{\tilde Q_t \left( \alpha \nabla \tilde f_t(x_t) \right)}^2
    =
    \norm{x_t - x^*}^2
    -
    2 (x_t - x^*)^T \tilde Q_t \left( \alpha \nabla \tilde f_t(x_t) \right)
    +
    \alpha^2 \norm{\alpha \nabla \tilde f_t(x_t)}^2
    +
    \norm{
      \tilde Q_t \left( \alpha \nabla \tilde f_t(x_t) \right)
      -
      \alpha \nabla \tilde f_t(x_t)
    }^2.
  \end{dmath*}
  Taking the expected value and applying (\ref{eqnConvexNoise}), and the
  bounds on the properties of $\tilde Q_t$, produces
  \begin{dmath*}
    \Exv{\norm{x_t + \tilde G_t(x_t) - x^*}^2}
    \le
    \norm{x_t - x^*}^2
    -
    2 \alpha (x_t - x^*)^T \nabla f(x_t)
    +
    \alpha^2 M^2
    +
    \delta^2.
  \end{dmath*}
  Since we assigned $\delta \le \alpha \kappa M$,
  \begin{dmath*}
    \Exv{\norm{x_t + \tilde G_t(x_t) - x^*}^2}
    \le
    \norm{x_t - x^*}^2
    -
    2 \alpha (x_t - x^*)^T \nabla f(x_t)
    +
    \alpha^2 M^2 (1 + \kappa^2)
    =
    \norm{x_t - x^*}^2
    -
    2 \alpha (x_t - x^*)^T \left( \nabla f(x_t) - \nabla f(x^*) \right)
    +
    \alpha^2 M^2 (1 + \kappa^2).
  \end{dmath*}
  Applying the strong convexity assumption (\ref{eqnStrongConvexity}),
  \begin{dmath*}
    \Exv{\norm{x_t + \tilde G_t(x_t) - x^*}^2}
    \le
    \norm{x_t - x^*}^2
    -
    2 \alpha c \norm{x_t - x^*}^2
    +
    \alpha^2 M^2 (1 + \kappa^2)
    =
    (1 - 2 \alpha c) \norm{x_t - x^*}^2
    +
    \alpha^2 M^2 (1 + \kappa^2).
  \end{dmath*}
  Now, if we haven't succeeded yet, then $\norm{x_t - x^*}^2 > \epsilon$.
  Under these conditions,
  \begin{dmath*}
    \Exv{\norm{x_t + \tilde G_t(x_t) - x^*}^2}
    \le
    \norm{x_t - x^*}^2 \left(
      1 
      - 
      2 \alpha c
      +
      \alpha^2 M^2 (1 + \kappa^2) \epsilon^{-1}
    \right).
  \end{dmath*}
  Multiplying both sides of the equation by $\epsilon^{-1}$ and taking the 
  piecewise logarithm, by Jensen's inequality
  \begin{dmath*}
    \Exv{\plog\left(
      \epsilon^{-1} \norm{x_t + \tilde G_t(x_t) - x^*}^2
    \right)}
    \le
    \plog\left(\Exv{
      \epsilon^{-1} \norm{x_t + \tilde G_t(x_t) - x^*}^2
    }\right)
    \le
    \plog\left(
      \epsilon^{-1} \norm{x_t - x^*}^2 \left(
        1 
        - 
        2 \alpha c
        +
        \alpha^2 M^2 (1 + \kappa^2) \epsilon^{-1}
      \right)
    \right).
  \end{dmath*}
  Since $\epsilon^{-1} \norm{x_t - x^*}^2 > 1$, we can apply Lemma
  \ref{lemmaPiecewiseLogarithm}, which gives us
  \begin{dmath*}
    \Exv{\plog\left(
      \epsilon^{-1} \norm{x_t + \tilde G_t(x_t) - x^*}^2
    \right)}
    \le
    \plog\left( \epsilon^{-1} \norm{x_t - x^*}^2 \right)
    -
    2 \alpha c
    +
    \alpha^2 M^2 (1 + \kappa^2) \epsilon^{-1}.
  \end{dmath*}
  Now, we define the rate supermartingale $W_t$ such that, if we haven't
  succeeded up to time $t$, then
  \[
    W_t(x_t, \ldots, x_0)
    =
    \frac{
      \epsilon
    }{
      2 \alpha c \epsilon
      -
      \alpha^2 M^2 (1 + \kappa^2)
    }
    \plog\left( \epsilon^{-1} \norm{x_t - x^*}^2 \right)
    +
    t;
  \]
  otherwise, if $u$ is a time such that $x_u \in S$, then for all $t > u$,
  \[
    W_t(x_t, \ldots, x_0) = W_u(x_u, \ldots, x_0).
  \]
  The first rate supermartingale property (\ref{eqnBoundedW}) is true because
  if success hasn't occurred,
  \begin{dmath*}
    \Exv{W_{t+1}(x_t + \tilde G_t(x_t), \ldots, x_0)}
    =
    \Exv{
      \frac{
        \epsilon
      }{
        2 \alpha c \epsilon
        -
        \alpha^2 M^2 (1 + \kappa^2)
      }
      \plog\left( \epsilon^{-1} \norm{x_t + \tilde G_t(x_t) - x^*}^2 \right)
      +
      (t + 1)
    }
    =
    \frac{
      \epsilon
    }{
      2 \alpha c \epsilon
      -
      \alpha^2 M^2 (1 + \kappa^2)
    }
    \Exv{
      \plog\left( \epsilon^{-1} \norm{x_t + \tilde G_t(x_t) - x^*}^2 \right)
    }
    +
    (t + 1)
    \le
    \frac{
      \epsilon
    }{
      2 \alpha c \epsilon
      -
      \alpha^2 M^2 (1 + \kappa^2)
    }
    \left(
      \plog\left( \epsilon^{-1} \norm{x_t - x^*}^2 \right)
      -
      2 \alpha c
      +
      \alpha^2 M^2 (1 + \kappa^2) \epsilon^{-1}
    \right)
    +
    (t + 1)
    =
    \frac{
      \epsilon
    }{
      2 \alpha c \epsilon
      -
      \alpha^2 M^2 (1 + \kappa^2)
    }
    \plog\left( \epsilon^{-1} \norm{x_t - x^*}^2 \right)
    -
    1
    +
    (t + 1)
    =
    W_t(x_t, \ldots, x_0);
  \end{dmath*}
  it is vacuously true if success has occurred because the value of $W_t$ does
  not change after $x_u \in S$ for $u < t$.
  The second rate supermartingale property (\ref{eqnBoundedTime}) holds
  because, if success hasn't occurred by time $T$,
  \[
    W_T(x_T, \ldots, x_0)
    =
    \frac{
      \epsilon
    }{
      2 \alpha c \epsilon
      -
      \alpha^2 M^2 (1 + \kappa^2)
    }
    \plog\left( \epsilon^{-1} \norm{x_T - x^*}^2 \right)
    +
    T
    \ge
    T;
  \]
  this follows from the non-negativity of the $\plog$ function for non-negative
  arguments.

  We have now shown that $W_t$ is a rate supermartingale for this algorithm.
  Next, we verify that the bound on $W_0$ given in the lemma statement holds.
  At time $0$, by the definition of the $\plog$ function, since we assume
  that success has not occurred yet,
  \begin{dmath*}
    W_0(x_0)
    =
    \frac{
      \epsilon
    }{
      2 \alpha c \epsilon
      -
      \alpha^2 M^2 (1 + \kappa^2)
    }
    \plog\left( \epsilon^{-1} \norm{x_0 - x^*}^2 \right)
    =
    \frac{
      \epsilon
    }{
      2 \alpha c \epsilon
      -
      \alpha^2 M^2 (1 + \kappa^2)
    }
    \log\left( \frac{e \norm{x_0 - x^*}^2}{\epsilon} \right);
  \end{dmath*}
  this is the bound given in the lemma statement.

  Next, we show that this rate supermartingale is $(H, R, \xi)$-bounded, for
  the values of $H$, $R$, and $\xi$ given in the lemma statement.
  First, for any $x$, $t$, and sequence $x_{t-1}, \ldots, x_0$,
  \begin{dmath*}
    \nabla_x W_t(x, x_{t-1}, \ldots, x_0)
    =
    \nabla_x \left(
      \frac{
        \epsilon
      }{
        2 \alpha c \epsilon
        -
        \alpha^2 M^2 (1 + \kappa^2)
      }
      \plog\left( \epsilon^{-1} \norm{x - x^*}^2 \right)
    \right)
    =
    \frac{
      \epsilon
    }{
      2 \alpha c \epsilon
      -
      \alpha^2 M^2 (1 + \kappa^2)
    }
    2 \epsilon^{-1} (x - x^*)
    \plog'\left( \epsilon^{-1} \norm{x - x^*}^2 \right).
  \end{dmath*}
  Now, by the definition of $\plog$, we can conclude that
  $\plog'(u) = \min\left(1, u^{-1}\right)$.  Therefore,
  \begin{dmath*}
    \nabla_x W_t(x, x_{t-1}, \ldots, x_0)
    =
    \frac{
      2
    }{
      2 \alpha c \epsilon
      -
      \alpha^2 M^2 (1 + \kappa^2)
    }
    (x - x^*)
    \min\left(
      1,
      \epsilon \norm{x - x^*}^{-2}
    \right),
  \end{dmath*}
  and taking the norm of both sides,
  \begin{dmath*}
    \nabla_x W_t(x, x_{t-1}, \ldots, x_0)
    =
    \frac{
      2
    }{
      2 \alpha c \epsilon
      -
      \alpha^2 M^2 (1 + \kappa^2)
    }
    \min\left(
      \norm{x - x^*},
      \epsilon \norm{x - x^*}^{-1}
    \right).
  \end{dmath*}
  Clearly, this expression is maximized when $\norm{x - x^*}^2 = \epsilon$.
  Therefore,
  \begin{dmath*}
    \nabla_x W_t(x, x_{t-1}, \ldots, x_0)
    \le
    \frac{
      2 \sqrt{\epsilon}
    }{
      2 \alpha c \epsilon
      -
      \alpha^2 M^2 (1 + \kappa^2)
    }.
  \end{dmath*}
  The Lipschitz continuity expression with $H$ in the lemma statement now
  follows from the mean value theorem.

  Next, we bound the Lipschitz continuity expression for $R$.  We have that,
  for any $x$ and $y$, if the single non-zero entry of $\nabla \tilde f$ is
  at index $i$, then
  \begin{dmath*}
    \Exv{\norm{\tilde G(x) - \tilde G(y)}}
    =
    \Exv{\norm{
      \tilde Q(\alpha \nabla \tilde f(x))
      -
      \tilde Q(\alpha \nabla \tilde f(y))
    }}
    =
    \Exv{\Abs{
      \tilde Q(\alpha e_i^T \nabla \tilde f(x))
      -
      \tilde Q(\alpha e_i^T \nabla \tilde f(y))
    }}
  \end{dmath*}
  Without loss of generality, we assume that $\tilde Q$ is non-decreasing, and
  that $e_i^T \nabla \tilde f(x) \ge e_i^T \nabla \tilde f(y)$.  Thus, by the
  unbiased quality of $\tilde Q$,
  \begin{dmath*}
    \Exv{\norm{\tilde G(x) - \tilde G(y)}}
    =
    \Exv{
      \tilde Q(e_i^T \alpha \nabla \tilde f(x))
      -
      \tilde Q(e_i^T \alpha \nabla \tilde f(y))
    }
    =
    \Exv{
      e_i^T \alpha \nabla \tilde f(x)
      -
      e_i^T \alpha \nabla \tilde f(y)
    }
    =
    \alpha \Exv{\norm{
      \nabla \tilde f(x) - \nabla \tilde f(y)
    }}.
  \end{dmath*}
  Finally, applying (\ref{eqnConvexLipschitz}),
  \begin{dmath*}
    \Exv{\norm{\tilde G(x) - \tilde G(y)}}
    \le
    \alpha L.
  \end{dmath*}

  Finally, we bound the update expression with $\xi$.  We have,
  \begin{dmath*}
    \Exv{\norm{\tilde G(x)}}^2
    =
    \Exv{\norm{
      \tilde Q(\alpha \nabla \tilde f(x))
    }}^2
    \le
    \Exv{\norm{
      \tilde Q(\alpha \nabla \tilde f(x))
    }^2}
    =
    \Exv{
      \alpha^2 \norm{\nabla \tilde f(x)}^2
      +
      2 \alpha (\nabla \tilde f(x))^T \left(
        \tilde Q(\alpha \nabla \tilde f(x))
        -
        \alpha \nabla \tilde f(x)
      \right)
      +
      \norm{
        \tilde Q(\alpha \nabla \tilde f(x))
        -
        \alpha \nabla \tilde f(x)
      }^2
    }.
  \end{dmath*}
  Applying the bounds on the rounding error,
  \begin{dmath*}
    \Exv{\norm{\tilde G(x)}}^2
    \le
    \Exv{
      \alpha^2 \norm{\nabla \tilde f(x)}^2
      +
      2 \alpha (\nabla \tilde f(x))^T \left(
        \tilde Q(\alpha \nabla \tilde f(x))
        -
        \alpha \nabla \tilde f(x)
      \right)
      +
      \delta^2
    }.
  \end{dmath*}
  Taking the expected value and applying (\ref{eqnConvexNoise}) and the
  unbiased quality of $\tilde Q$,
  \begin{dmath*}
    \Exv{\norm{\tilde G(x)}}^2
    \le
    \alpha^2 M^2
    +
    \delta^2.
  \end{dmath*}
  Applying the assignment $\delta = \alpha \kappa M$ results in
  \begin{dmath*}
    \Exv{\norm{\tilde G(x)}}^2
    \le
    \alpha^2 M^2 (1 + \kappa^2),
  \end{dmath*}
  which is the desired expression.

  So, we have proved all the statements in the lemma.
\end{proof}

\begin{proof}[Proof of Corollary \ref{corConvexBuckwild}]
Applying Theorem \ref{thmHogwild} directly to the result of Lemma
\ref{lemmaConvexW} produces
\begin{dmath*}
  \Prob{F_T}
  \le
  \frac{
    \Exv{W_0(x_0)}
  }{
    (1 - H R \xi \tau) T
  }
  =
  \frac{
    \epsilon
  }{
    2 \alpha c \epsilon
    -
    \alpha^2 M^2 (1 + \kappa^2)
  }
  \log\left( \frac{e \norm{x_0 - x^*}^2}{\epsilon} \right)
  \left(
    \left(
      1
      - 
      \left(
        \frac{
          2 \sqrt{\epsilon}
        }{
          2 \alpha c \epsilon
          -
          \alpha^2 M^2 (1 + \kappa^2)
        }
      \right)
      (\alpha L) (\alpha M \sqrt{1 + \kappa^2}) \tau
    \right) T
  \right)^{-1}
  =
  \frac{
    \epsilon
  }{
    \left(
      2 \alpha c \epsilon
      -
      \alpha^2 \left(
        M^2 (1 + \kappa^2)
        -
        2 L M \tau \sqrt{1 + \kappa^2} \sqrt{\epsilon}
      \right)
    \right)
    T
  }
  \log\left( \frac{e \norm{x_0 - x^*}^2}{\epsilon} \right)
  \le
  \frac{
    \epsilon
  }{
    \left(
      2 \alpha c \epsilon
      -
      \alpha^2 \left(
        M^2 (1 + \kappa^2)
        -
        L M \tau (2 + \kappa^2) \sqrt{\epsilon}
      \right)
    \right)
    T
  }
  \log\left( \frac{e \norm{x_0 - x^*}^2}{\epsilon} \right)
\end{dmath*}
Substituting the chosen value of $\alpha$,
\begin{dmath*}
  \Prob{F_T}
  \le
  \frac{
    \epsilon
  }{
    T
  }
  \left(
    2 c \epsilon
    \left(
      \frac{
        c \epsilon \vartheta
      }{
        M^2 (1 + \kappa^2)
        +
        L M \tau (2 + \kappa^2) \sqrt{\epsilon}
      }
    \right)
    -
    \left(
      M^2 (1 + \kappa^2)
      -
      L M \tau (2 + \kappa^2) \sqrt{\epsilon}
    \right)
    \left(
      \frac{
        c \epsilon \vartheta
      }{
        M^2 (1 + \kappa^2)
        +
        L M \tau (2 + \kappa^2) \sqrt{\epsilon}
      }
    \right)^2
  \right)^{-1}
  \log\left( \frac{e \norm{x_0 - x^*}^2}{\epsilon} \right)
  =
  \frac{
    \epsilon
  }{
    \left(
      \frac{
        2 c^2 \epsilon^2 \vartheta
      }{
        M^2 (1 + \kappa^2) + L M \tau (2 + \kappa^2) \sqrt{\epsilon}
      }
      -
      \frac{
        c^2 \epsilon^2 \vartheta^2
      }{
        M^2 (1 + \kappa^2) + L M \tau (2 + \kappa^2) \sqrt{\epsilon}
      }
    \right)
    T
  }
  \log\left( \frac{e \norm{x_0 - x^*}^2}{\epsilon} \right)
  \le
  \frac{
    \epsilon
  }{
    \frac{
      c^2 \epsilon^2 \vartheta
    }{
      M^2 (1 + \kappa^2) + L M \tau (2 + \kappa^2) \sqrt{\epsilon}
    }
    T
  }
  \log\left( \frac{e \norm{x_0 - x^*}^2}{\epsilon} \right)
  =
  \frac{
    M^2 (1 + \kappa^2) + L M \tau (2 + \kappa^2) \sqrt{\epsilon}
  }{
    c^2 \epsilon \vartheta T
  }
  \log\left( \frac{e \norm{x_0 - x^*}^2}{\epsilon} \right),
\end{dmath*}
as desired.
\end{proof}

\section{Proofs for Non-Convex Case}
In order to accomplish this proof, we make use of some definitions and
lemmas that appear in \citet{desa2014global}.  We state them here before
proceeding to the proof.


First, we define a function
\[
  \tau(x)
  =
  \frac{
    (u_1^T x)^2
  }{
    (1 - \gamma n^{-1}) (u_1^T x)^2
    + 
    \gamma n^{-1} \norm{x}^2
  }.
\]
Clearly, $0 \le \tau(x) \le 1$.  Using this function,
\citet{desa2014global} prove the following lemma.  While their version of the
lemma applies to higher-rank problems and multiple distributions, we state
here a version that is specialized for the rank-1, entrywise sampling case
we study in this paper. (This is a combination of Lemma 2 and Lemma 12
from \citet{desa2014global}.)

\begin{lemma}[$\tau$-bound]
\label{lemmaAlectonTauBound}
If we run the Alecton update rule using entrywise sampling under the conditions
in Section \ref{ssNonConvex}, including the incoherence and 
step size assignment, then for any $x \notin S$,
\[
  \Exv{\tau(x + \eta \tilde A x)}
  \ge
  \tau(x) \left(
    1
    +
    \eta \Delta (1 - \tau(x))
  \right).
\]
\end{lemma}

We also use another lemma from \citet{desa2014global}.  This is a combination
of their Lemmas 1 and 7. 

\begin{lemma}[Expected value of $\tau(x_0)$]
\label{lemmaAlectonTau0}
If we initialize $x_0$ with a uniform random angle (as done in Alecton), then
\[
  \Exv{1 - \tau(x_0)} \le \sqrt{\frac{\pi \gamma}{2}}.
\]
\end{lemma}

Now, we prove Lemma \ref{lemmaAlectonW}.

\begin{proof}[Proof of Lemma \ref{lemmaAlectonW}]
First, if $x \notin S$, then $(u_1^T x)^2 \le (1 - \epsilon) \norm{x}^2$. 
Therefore,
\begin{dmath*}
  \tau(x)
  =
  \frac{
    (u_1^T x)^2
  }{
    (1 - \gamma n^{-1}) (u_1^T x)^2
    + 
    \gamma n^{-1} \norm{x}^2
  }
  \le
  \frac{
    1 - \epsilon
  }{
    (1 - \gamma n^{-1}) (1 - \epsilon)
    + 
    \gamma n^{-1}
  }
  =
  \frac{
    1 - \epsilon
  }{
    1
    -
    \epsilon
    +
    \gamma n^{-1} \epsilon,
  }
\end{dmath*}
and so
\[
  1 - \tau(x)
  \ge
  \frac{
    \gamma n^{-1} \epsilon
  }{
    1
    -
    \epsilon
    +
    \gamma n^{-1} \epsilon,
  }
  >
  \gamma n^{-1} \epsilon.
\]
From the result of Lemma \ref{lemmaAlectonTauBound}, for any
$x \notin S$,
\[
  \Exv{\tau(x + \eta \tilde A x)}
  \ge
  \tau(x) \left(
    1
    +
    \eta \Delta (1 - \tau(x))
  \right).
\]
Therefore,
\[
  \Exv{1 - \tau(x + \eta \tilde A x)}
  \le
  (1 - \tau(x)) \left(
    1
    -
    \eta \Delta \tau(x)
  \right)
\]
Therefore, by Jensen's inequality and Lemma \ref{lemmaPiecewiseLogarithm},
since $\gamma^{-1} n \epsilon (1 - \tau(x)) > 1$,
\begin{dmath*}
  \Exv{
    \plog\left( 
      \gamma^{-1} n \epsilon^{-1} 
      \left( 1 - \tau(x + \eta \tilde A x) \right)
    \right)
  }
  \ge
  \plog\left( 
    \Exv{
      \gamma^{-1} n \epsilon^{-1} 
      \left( 1 - \tau(x + \eta \tilde A x) \right)
    }
  \right)
  \ge
  \plog \left( \gamma^{-1} n \epsilon^{-1} (1 - \tau(x)) \left(
    1
    -
    \eta \Delta \tau(x)
  \right) \right)
  \ge
  \plog \left( \gamma^{-1} n \epsilon^{-1} (1 - \tau(x)) \right)
  -
  \eta \Delta \tau(x).
\end{dmath*}

Now, we define our rate supermartingale.  First, define
\[
  Z = \left\{ x \middle| \tau(x) \ge \frac{1}{2} \right\},
\]
and let $B > 0$ be any constant.
Let $W_t$ be defined such that, if $x_u \notin S \cup Z$ for all $u \le t$,
then
\[
  W_t(x_t, \ldots, x_0)
  =
  \frac{2}{\eta \Delta}
  \plog \left( \gamma^{-1} n \epsilon^{-1} (1 - \tau(x_t)) \right)
  +
  2B (1 - \tau(x_t))
  +
  t.
\]
On the other hand, if $x_u \in S \cup Z$ for some $u$, then for all $t > u$,
we define
\[
  W_t(x_t, \ldots, x_0) = W_u(x_u, \ldots, x_0).
\]
That is, once $x_t$ enters $S \cup Z$, the process $W$ stops changing.

We verify that $W_t$ is a rate supermartingale.  First,
(\ref{eqnBoundedW}) is true because, in the case that the process has stopped
it is true vacuously, and in the case that it hasn't stopped
(i.e. $x_i \notin S \cup Z$ for all $u \le t$),
\begin{dmath*}
  \Exv{W_{t+1}(x_t + \eta \tilde A_t x_t, x_t, \ldots, x_0}
  =
  \Exv{
    \frac{2}{\eta \Delta}
    \plog \left( 
      \gamma^{-1} n \epsilon^{-1} 
      (1 - \tau(x_t + \eta \tilde A_t x_t))
    \right)
    +
    2B (1 - \tau(x_t + \eta \tilde A_t x_t))
    +
    t + 1
  }
  =
  \frac{2}{\eta \Delta}
  \Exv{
    \plog \left( 
      \gamma^{-1} n \epsilon^{-1} 
      (1 - \tau(x_t + \eta \tilde A_t x_t))
    \right)
  }
  +
  2B \Exv{1 - \tau(x_t + \eta \tilde A_t x_t)}
  +
  t + 1
  \le
  \frac{2}{\eta \Delta}
  \left(
    \plog \left( \gamma^{-1} n \epsilon^{-1} (1 - \tau(x_t)) \right)
    -
    \eta \Delta \tau(x_t)
  \right)
  +
  2B (1 - \tau(x_t))
  +
  t + 1
  =
  W_t(x_t, \ldots, x_0)
  -
  2 \tau(x_t)
  +
  1.
\end{dmath*}
Since $x_t \notin Z$, it follows that $2 \tau(x_t) \ge 1$.  Therefore,
\[
  \Exv{W_{t+1}(x_t + \eta \tilde A_t x_t, x_t, \ldots, x_0}
  \le
  W_t(x_t, \ldots, x_0).
\]
And so (\ref{eqnBoundedW}) holds in all cases.

The second rate supermartingale property (\ref{eqnBoundedTime}) holds
because, if success hasn't occurred by time $T < B$, then there are two
possibilities: either the process hasn't stopped yet, or it stopped at a
timestep where $x_t \in Z$.  In the former case, by the non-negativity of
the $\plog$ function,
\[
  W_T(x_T, \ldots, x_0)
  =
  \frac{2}{\eta \Delta}
  \plog \left( \gamma^{-1} n \epsilon^{-1} (1 - \tau(x_T)) \right)
  +
  2B (1 - \tau(x_T))
  +
  T
  \ge
  T.
\]
In the latter case,
\begin{dmath*}
  W_T(x_T, \ldots, x_0)
  =
  \frac{2}{\eta \Delta}
  \plog \left( \gamma^{-1} n \epsilon^{-1} (1 - \tau(x_T)) \right)
  +
  2B (1 - \tau(x_T))
  +
  T
  \ge
  B.
\end{dmath*}
Therefore (\ref{eqnBoundedTime}) holds.

We have now shown that $W_t$ is a rate supermartingale for Alecton.
Next, we show that our bound on the initial value of the supermartingale holds.
At time $0$,
\begin{dmath*}
  W_0(x_0)
  =
  \frac{2}{\eta \Delta}
  \plog \left( \gamma^{-1} n \epsilon^{-1} (1 - \tau(x_0)) \right)
  +
  2B (1 - \tau(x_0))
  \le
  \frac{2}{\eta \Delta}
  \plog \left( \gamma^{-1} n \epsilon^{-1} \right)
  +
  2B (1 - \tau(x_0))
  =
  \frac{2}{\eta \Delta}
  \log \left( \frac{e n}{\gamma \epsilon} \right)
  +
  2B (1 - \tau(x_0)).
\end{dmath*}
Therefore, applying Lemma \ref{lemmaAlectonTau0},
\begin{dmath*}
  \Exv{W_0(x_0)}
  \le
  \frac{2}{\eta \Delta}
  \log \left( \frac{e n}{\gamma \epsilon} \right)
  +
  2B \Exv{1 - \tau(x_0)}
  \le
  \frac{2}{\eta \Delta}
  \log \left( \frac{e n}{\gamma \epsilon} \right)
  +
  B \sqrt{2 \pi \gamma}.
\end{dmath*}
This is the value given in the lemma.

Now, we show that $W_t$ is $(H, R, \xi)$-bounded.  First, we give the $H$
bound.  To do so, we first differentiate $\tau(x)$.
\begin{dmath*}
  \nabla \tau(x)
  =
  \frac{
    2 u_1 u_1^T x
    \left(
      (1 - \gamma n^{-1}) (u_1^T x)^2
      + 
      \gamma n^{-1} \norm{x}^2
    \right)
    -
    2 (u_1^T x)^2
    \left(
      (1 - \gamma n^{-1}) u_1 u_1^T x
      + 
      \gamma n^{-1} x
    \right)
  }{
    \left(
      (1 - \gamma n^{-1}) (u_1^T x)^2
      + 
      \gamma n^{-1} \norm{x}^2
    \right)^2
  }
  =
  \frac{
    2 u_1 u_1^T x \gamma n^{-1} \norm{x}^2
    - 
    2 (u_1^T x)^2 \gamma n^{-1} x
  }{
    \left(
      (1 - \gamma n^{-1}) (u_1^T x)^2
      + 
      \gamma n^{-1} \norm{x}^2
    \right)^2
  }
  =
  2 \gamma n^{-1}
  \frac{
    u_1 u_1^T x \norm{x}^2
    - 
    x (u_1^T x)^2
  }{
    \left(
      (1 - \gamma n^{-1}) (u_1^T x)^2
      + 
      \gamma n^{-1} \norm{x}^2
    \right)^2
  }.
\end{dmath*}
Therefore,
\begin{dmath*}
  \norm{\nabla \tau(x)}^2
  =
  4 \gamma^2 n^{-2}
  \frac{
    (u_1^T x)^2 \norm{x}^4
    -
    (u_1^T x)^4 \norm{x}^2
  }{
    \left(
      (1 - \gamma n^{-1}) (u_1^T x)^2
      + 
      \gamma n^{-1} \norm{x}^2
    \right)^4
  }
  \le
  4 \gamma^2 n^{-2}
  \frac{
    \norm{x}^4
    -
    (u_1^T x)^2 \norm{x}^2
  }{
    \left(
      (1 - \gamma n^{-1}) (u_1^T x)^2
      + 
      \gamma n^{-1} \norm{x}^2
    \right)^3
  }
  \le
  4 \gamma n^{-1}
  \frac{
    \norm{x}^2 (1 - \tau(x))
  }{
    \left(
      (1 - \gamma n^{-1}) (u_1^T x)^2
      + 
      \gamma n^{-1} \norm{x}^2
    \right)^2
  }
  \le
  \frac{
    4 (1 - \tau(x))
  }{
    \left(
      (1 - \gamma n^{-1}) (u_1^T x)^2
      + 
      \gamma n^{-1} \norm{x}^2
    \right)
  }
  \le
  \frac{
    4 n (1 - \tau(x))
  }{
    \gamma \norm{x}^2
  }.
\end{dmath*}
Applying the assumption that $\norm{x}^2 \ge 1$,
\begin{dmath*}
  \norm{\nabla \tau(x)}
  \le
  \sqrt{
    \frac{
      4 n (1 - \tau(x))
    }{
      \gamma
    }
  }.
\end{dmath*}
Now, differentiating $W_t$ with respect to $\tau$ produces
\begin{dmath*}
  \frac{dW}{d \tau}
  =
  -
  \frac{2 n}{\eta \gamma \epsilon \Delta}
  \plog' \left( \gamma^{-1} n \epsilon^{-1} (1 - \tau) \right)
  -
  2B.
\end{dmath*}
So, it follows that
\begin{dmath*}
  \norm{\nabla_x W_t(x, x_{t-1}, \ldots, x_0)}
  \le
  \Abs{\frac{dW}{d \tau}}
  \norm{\nabla \tau(x)}
  \le
  \left(
    \frac{2 n}{\eta \gamma \epsilon \Delta}
    \plog' \left( \gamma^{-1} n \epsilon^{-1} (1 - \tau) \right)
    +
    2B
  \right)
  \sqrt{
    \frac{
      4 n (1 - \tau(x))
    }{
      \gamma
    }
  }.
\end{dmath*}
Applying our assumption that $\eta \gamma \epsilon \Delta B \le 1$, 
it is clear that this function will
be maximized when $\gamma^{-1} n \epsilon^{-1} (1 - \tau) = 1$.  Therefore,
\begin{dmath*}
  \norm{\nabla_x W_t(x, x_{t-1}, \ldots, x_0)}
  \le
  \left(
    \frac{2 n}{\eta \gamma \epsilon \Delta}
    +
    2B
  \right)
  2 \sqrt{\epsilon}
  =
  \frac{8 n}{\eta \gamma \Delta \sqrt{\epsilon}},
\end{dmath*}
which is our given value for $H$.

Next, we give the $R$ bound.  For Alecton, we have
\[
  \tilde G(x) = \eta \tilde A x = \eta n^2 e_i e_i^T A e_j e_j^T x.
\]
Therefore,
\begin{dmath*}
  \Exv{\norm{\tilde G(x) - \tilde G(y)}}
  =
  \eta n^2 \Exv{\norm{e_i e_i^T A e_j e_j^T (x - y)}}
  =
  \eta n^2 \Exv{\Abs{e_i^T A e_j e_j^T (x - y)}}
  =
  \eta \sum_{i=1}^n \sum_{j=1}^n
  \Abs{e_i^T A e_j} \Abs{e_j^T (x - y)}
  =
  \eta \sum_{j=1}^n \Abs{e_j^T (x - y)} \left(
    \sum_{i=1}^n \Abs{e_i^T A e_j}
  \right)
  \le
  \eta \sum_{j=1}^n \Abs{e_j^T (x - y)} \sqrt{n} \left(
    \sum_{i=1}^n (e_i^T A e_j)^2
  \right)^{\frac{1}{2}}
  =
  \eta \sum_{j=1}^n \Abs{e_j^T (x - y)} \sqrt{n} \left(
    e_j^T A^2 e_j
  \right)^{\frac{1}{2}}
  =
  \eta \sum_{j=1}^n \Abs{e_j^T (x - y)} \sqrt{n} \left(
    \sum_{k=1}^{\infty} \lambda_j^2 (u_k^T e_j)^2
  \right)^{\frac{1}{2}}.
\end{dmath*}
Applying the incoherence bound,
\begin{dmath*}
  \Exv{\norm{\tilde G(x) - \tilde G(y)}}
  \le
  \eta \sum_{j=1}^n \Abs{e_j^T (x - y)} \sqrt{n} \left(
    \sum_{k=1}^{\infty} \lambda_j^2 \mu^2 n^{-1}
  \right)^{\frac{1}{2}}
  =
  \eta \sum_{j=1}^n \Abs{e_j^T (x - y)} \sqrt{n} \left(
    \mu^2 n^{-1} \normf{A}^2
  \right)^{\frac{1}{2}}
  =
  \eta \sum_{j=1}^n \Abs{e_j^T (x - y)} \mu \normf{A}
  =
  \eta \mu \normf{A} \norm{x - y}_1.
\end{dmath*}
This agrees with our assignment of $R = \eta \mu \normf{A}$.

Finally, we give our $\xi$ bound on the magnitude of the updates.
By the same argument as above, we will have
\begin{dmath*}
  \Exv{\norm{\tilde G(x)}}
  =
  \eta n^2 \Exv{\norm{e_i e_i^T A e_j e_j^T x}}
  =
  \eta \mu \normf{A} \norm{x}_1.
\end{dmath*}
Applying the assumption that $\norm{x}_1^2 \le C$, produces the bound given
in the lemma, $\xi = \eta \mu \normf{A} C$.

This completes the proof of the lemma.
\end{proof}

Next, we prove the corollary that gives a bound on the failure probability
of asynchronous Alecton.
\begin{proof}[Proof of Corollary \ref{thmHogwildAlecton}]
By Theorem \ref{thmHogwild}, we know that
for the constants defined in Lemma \ref{lemmaAlectonW},
\[
  \Prob{F_T}
  \le
  \frac{
    \Exv{W(0, x_0)}
  }{
    (1 - H R \xi \tau) T
  }.
\]
If we choose $B = T$ for the horizon in Lemma \ref{lemmaAlectonW}, and
substitute in the given constants,
\begin{dmath*}
  \Prob{F_T}
  \le
  \left(
    \frac{2}{\eta \Delta}
    \log \left(
      \frac{e n}{\gamma \epsilon}
    \right)
    +
    T \sqrt{2 \pi \gamma}
  \right)
  \left(
    1
    -
    \left(
      \frac{8 n}{\eta \gamma \Delta \sqrt{\epsilon}}
    \right)
    \left(
      \eta \mu \normf{A}
    \right)
    \left(
      \eta \mu \normf{A} C
    \right)
    \tau
  \right)^{-1}
  T^{-1}
  =
  \left(
    \frac{2}{\eta \Delta T}
    \log \left(
      \frac{e n}{\gamma \epsilon}
    \right)
    +
    \sqrt{2 \pi \gamma}
  \right)
  \left(
    1
    -
    \frac{8 \eta n \mu^2 \normf{A}^2 C \tau}{\gamma \Delta \sqrt{\epsilon}}
  \right)^{-1}.
\end{dmath*}
Now, for the given value of $\eta$, we will have
\begin{dmath*}
  \frac{8 \eta n \mu^2 \normf{A}^2 C \tau}{\gamma \Delta \sqrt{\epsilon}}
  =
  \frac{
    \Delta \epsilon \gamma \vartheta
  }{
    2 n \mu^4 \normf{A}^2
  }
  \frac{8 n \mu^2 \normf{A}^2 C \tau}{\gamma \Delta \sqrt{\epsilon}}
  =
  \frac{4 C \vartheta \tau \sqrt{\epsilon}}{\mu^2}.
\end{dmath*}
Also, for the given values of $\eta$ and $T$, we will have
\begin{dmath*}
  \frac{2}{\eta \Delta T}
  \log \left(
    \frac{e n}{\gamma \epsilon}
  \right)
  =
  \frac{
    2 n \mu^4 \normf{A}^2
  }{
    \Delta \epsilon \gamma \vartheta
  }
  \frac{
    \Delta^2 \epsilon \gamma \vartheta \sqrt{2 \pi \gamma}
  }{
    4 n \mu^4 \normf{A}^2
  }
  \frac{2}{\Delta}
  =
  \sqrt{2 \pi \gamma}.
\end{dmath*}
Substituting these results in produces
\begin{dmath*}
  \Prob{F_T}
  \le
  \sqrt{8 \pi \gamma}
  \left(
    1
    -
    \frac{4 C \vartheta \tau \sqrt{\epsilon}}{\mu^2}
  \right)^{-1}
  =
  \frac{
    \sqrt{8 \pi \gamma} \mu^2
  }{
    \mu^2 - 4 C \vartheta \tau \sqrt{\epsilon}
  },
\end{dmath*}
which is the desired result.
\end{proof}

\section{Simplified Convex Result}
\label{ssSimpleProofHogwild}
In this section, we provide a simplified proof for a result similar to our
main result that only
works in the convex case.  This proof does not use any martingale results, and
can therefore be considered more elementary than the proofs given above;
however, it does not generalize to the non-convex case.
\begin{theorem}
  \label{thmHogwildConvex}
  Under the conditions given in Section \ref{ssConvexHighPrecision},
  for any $\epsilon > 0$,
  if for some $\vartheta \in (0, 1)$ we choose constant step size
  \[
    \alpha
    =
    \frac{
      c \vartheta \epsilon
    }{
      2 L M \tau \sqrt{\epsilon} + M^2
    },
  \]
  then there exists a timestep
  \[
    T
    \le
    \frac{
      2 L M \tau \sqrt{\epsilon} + M^2
    }{
      c^2 \vartheta \epsilon
    }
    \log \left( \frac{\norm{x_0 - x^*}^2}{\epsilon} \right)
  \]
  such that
  \[
    \Exv{\norm{x_T - x^*}^2} \le \epsilon.
  \]
\end{theorem}
\begin{proof}
Our goal is to
bound the square-distance to the optimum by showing that it generally decreases
at each timestep.  We can show algebraically that
\begin{dmath*}
  \norm{x_{t+1} - x^*}^2
  =
  \norm{x_t - x^*}^2
  -
  2 \alpha (x_t - x^*)^T \nabla \tilde f_t(x_t)
  +
  2 \alpha (x_t - x^*)^T \left(
    \nabla \tilde f_t(x_t)
    -
    \nabla \tilde f_t(\tilde v_t)
  \right)
  +
  \alpha^2 \norm{\nabla \tilde f_t(\tilde v_t)}^2.
\end{dmath*}
We can think of these terms as representing respectively: the current
square-distance, the first-order change, the noise due to delayed updates, and
the noise due to random sampling.  Taking the expected value given
$\tilde v_t$ and applying Cauchy-Schwarz, 
(\ref{eqnStrongConvexity}), (\ref{eqnConvexLipschitz}),
and (\ref{eqnConvexNoise}) produces
\begin{dmath*}
  \Exvc{\norm{x_{t+1} - x^*}^2}{\F_t, \tilde v_t}
  \le
  \norm{x_t - x^*}^2
  -
  2 \alpha c \norm{x_t - x^*}^2
  +
  2 \alpha L \norm{x_t - x^*} \norm{x_t - \tilde v_t}_1
  +
  \alpha^2 M^2
  =
  (1 - 2 \alpha c) \norm{x_t - x^*}^2
  +
  \alpha^2 M^2
  +
  2 \alpha L \norm{x_t - x^*}
  \sum_{i=1}^n
  \Abs{e_i^T x_t - e_i^T x_{t - \tilde \tau_{i,t}}}
  \le
  (1 - 2 \alpha c) \norm{x_t - x^*}^2
  +
  \alpha^2 M^2
  +
  2 \alpha L \norm{x_t - x^*}
  \sum_{i=1}^n
  \sum_{k=1}^{\tilde \tau_{i,t}}
  \Abs{e_i^T x_{t-k+1} - e_i^T x_{t-k}}.
\end{dmath*}
We can now take the full expected value given the filtration,
which produces
\begin{dmath*}
  \Exvc{\norm{x_{t+1} - x^*}^2}{\F_t}
  \le
  (1 - 2 \alpha c) \norm{x_t - x^*}^2
  +
  \alpha^2 M^2
  +
  2 \alpha L \norm{x_t - x^*}
  \sum_{i=1}^n
  \sum_{k=1}^{\infty}
  \Prob{\tilde \tau_{i,k} \ge k}
  \Abs{e_i^T x_{t-k+1} - e_i^T x_{t-k}}.
\end{dmath*}
Applying (\ref{eqnHogwildTauBound}) results in
\begin{dmath*}
  \Exvc{\norm{x_{t+1} - x^*}^2}{\F_t}
  \le
  (1 - 2 \alpha c) \norm{x_t - x^*}^2
  +
  \alpha^2 M^2
  +
  2 \alpha L \norm{x_t - x^*}
  \sum_{i=1}^n
  \sum_{k=1}^{\infty}
  \Prob{\tilde \tau \ge k}
  \Abs{e_i^T x_{t-k+1} - e_i^T x_{t-k}}
  =
  (1 - 2 \alpha c) \norm{x_t - x^*}^2
  +
  \alpha^2 M^2
  +
  2 \alpha L \norm{x_t - x^*}
  \sum_{k=1}^{\infty}
  \Prob{\tilde \tau \ge k}
  \norm{x_{t-k+1} - x_{t-k}}_1,
\end{dmath*}
and since only at most one entry of $x$ changes at each iteration,
\begin{dmath*}
  \Exvc{\norm{x_{t+1} - x^*}^2}{\F_t}
  \le
  (1 - 2 \alpha c) \norm{x_t - x^*}^2
  +
  \alpha^2 M^2
  +
  2 \alpha L
  \sum_{k=1}^{\infty}
  \Prob{\tilde \tau \ge k}
  \norm{x_t - x^*}
  \norm{x_{t-k+1} - x_{t-k}}.
\end{dmath*}
Finally, taking the full expected value, and applying Cauchy-Schwarz again,
\begin{dmath*}
  \Exv{\norm{x_{t+1} - x^*}^2}
  \le
  (1 - 2 \alpha c) \Exv{\norm{x_t - x^*}^2}
  +
  \alpha^2 M^2
  +
  2 \alpha L
  \sum_{k=1}^{\infty}
  \Prob{\tilde \tau \ge k}
  \sqrt{
    \Exv{\norm{x_t - x^*}^2}
    \Exv{\norm{x_{t-k+1} - x_{t-k}}^2}
  }.
\end{dmath*}
Noticing that, from (\ref{eqnConvexNoise}),
\[
  \Exv{\norm{x_{t-k+1} - x_{t-k}}^2}
  =
  \Exv{\norm{\alpha \tilde G(\tilde v_{t-k})}^2}
  \le
  \alpha^2 M,
\]
if we let $J_t = \Exv{\norm{x_t - x^*}^2}$, we get
\begin{dmath*}
  J_{t+1}
  \le
  (1 - 2 \alpha c) J_t
  +
  \alpha^2 M^2
  +
  2 \alpha^2 L M
  \sum_{k=1}^{\infty}
  \Prob{\tilde \tau \ge k}
  \sqrt{
    J_t
  }
  =
  (1 - 2 \alpha c) J_t
  +
  \alpha^2 M^2
  +
  2 \alpha^2 L M \tau \sqrt{J_t}.
\end{dmath*}
For any $\epsilon > 0$, as long as $J_t \ge \epsilon$,
\begin{dmath*}
  \log J_{t+1}
  \le
  \log J_t
  +
  \log \left(
    1 - 2 \alpha c
    +
    \alpha^2 M^2 \epsilon^{-1}
    +
    2 \alpha^2 L M \tau \epsilon^{-\frac{1}{2}}
  \right)
  <
  \log J_t
  -
  2 \alpha c
  +
  \alpha^2 M^2 \epsilon^{-1}
  +
  2 \alpha^2 L M \tau \epsilon^{-\frac{1}{2}}.
\end{dmath*}
If we substitute the value of $\alpha$ chosen in the theorem statement, then
\begin{dmath*}
  \log J_{t+1}
  <
  \log J_t
  -
  \frac{
    c^2 \vartheta \epsilon
  }{
    2 L M \tau \sqrt{\epsilon} + M^2
  }.
\end{dmath*}
Therefore, for any $T$, if $J_T \ge \epsilon$ for all $t < T$,
\begin{dmath*}
  T
  <
  \frac{
    2 L M \tau \sqrt{\epsilon} + M^2
  }{
    c^2 \vartheta \epsilon
  }
  \log \left( \frac{J_0}{J_T} \right),
\end{dmath*}
which proves the theorem.
\end{proof}


}{}

\end{document}